\documentclass[11pt]{article}%
\usepackage{amsfonts}
\usepackage{amssymb}
\usepackage{graphicx}
\usepackage{setspace}
\usepackage[round]{natbib}
\usepackage{amsmath}%
\usepackage{amsthm}%
\usepackage{palatino}
\usepackage{array}
\newcolumntype{C}{>{\centering\arraybackslash}p{2.2cm}}
\usepackage{paralist}
\usepackage[top=8pc,bottom=8pc,left=8pc,right=8pc]{geometry}

\newtheorem{theorem}{Theorem}

\newtheorem{proposition}[theorem]{Proposition}

\newcommand{\prox}{ \mathop{\mathrm{prox}} }

\newcommand{\enorm}[1]{\Vert #1 \Vert_2}
\numberwithin{equation}{section}
\theoremstyle{plain}
\newtheorem{thm}{Theorem}[section]

\newcommand{\defeq}{\mathrel{\mathop:}=}

\title{Bayesian $l_0$ Regularized Least Squares}

\author{Nicholas G. Polson\footnote{Nicholas G. Polson is Professor of Econometrics and Statistics at the University of Chicago Booth School of Business. email: ngp@chicagobooth.edu}\\
\textit{Booth School of Business}\\
\textit{University of Chicago}
\and
Lei Sun\footnote{Lei Sun is PhD Candidate at the Department of Statistics, University of Chicago. email: sunl@uchicago.edu}\\
\textit{Department of Statistics}\\
\textit{University of Chicago}\\
\\
}

\date{First Draft: April 30, 2017\\
This Draft: Aug 12, 2018}

\begin{document}

\maketitle
\begin{abstract}
\noindent
Bayesian $l_0$-regularized least squares is a variable selection technique for high dimensional predictors.  The challenge is optimizing a non-convex objective function via search over model space consisting of all possible predictor combinations.  Spike-and-slab (a.k.a. Bernoulli-Gaussian) priors are the gold standard for Bayesian variable selection, with a caveat of computational speed and scalability. Single Best Replacement (SBR) provides a fast scalable alternative. We provide a link between Bayesian regularization and proximal updating, which provides an equivalence between finding a posterior mode and a posterior mean with a different regularization prior. This allows us to use SBR to find the spike-and-slab estimator. To illustrate our methodology, we provide simulation evidence and a real data example on the statistical properties and computational efficiency of SBR versus direct posterior sampling using spike-and-slab priors.  Finally, we conclude with directions for future research.

\vspace{0.5pc}

\end{abstract}

\newpage

\section{Introduction}

Bayesian $l_0$ regularization is an attractive solution for high dimensional variable selection as it directly penalizes the number of predictors.  The caveat is the need to search over all possible model combinations, as a full solution requires enumeration over all possible models which is NP-hard.  The gold standard for Bayesian variable selection are spike-and-slab priors, or Bernoulli-Gaussian mixtures \citep{mitchell1988, george1993, scott2014}.  Whilst spike-and-slab priors provide full model uncertainty quantification, they can be hard to scale to very high dimensional problems, and can have poor sparsity properties \citep{amini2012}.  On the other hand, techniques like proximal algorithms \citep{polson2015proximal, polson2017proximal} can solve non-convex optimization problems which are fast and scalable, but they generally don't provide a full assessment of model uncertainty \citep{jeffreys1961,hans2009,scott2010,li2010bayesian,marjanovic2013}.

Our goal is to build on the single best replacement (SBR) algorithm of \cite{soussen2011} as one approach to Bayesian $l_0$ regularization, in contrast with other methods based on proximal algorithms \citep{parikh2014proximal,polson2015proximal,polson2017proximal}. Our approach also builds on other Bayesian regularization methods including, for example, the Bayesian bridge \citep{polson2014}, horseshoe regularization \citep{carvalho2010,bhadra2016}, SVMs \citep{polson2011}, Bayesian lasso \citep{hans2009,park2008,carlin1991}, Bayesian elastic net \citep{li2010bayesian,hans2011}, spike-and-slab lasso \citep{rockova2016}, and global-local shrinkage priors \citep{bhadra2016,griffin2010}.

To fix notation, statistical regularization requires the specification of a measure of fit, denoted by $l\left(\beta\right)$ and a penalty function, denoted by $\text{pen}_\lambda\left(\beta\right)$, where $\lambda$ is a global regularization parameter.  From a Bayesian perspective,  $l\left(\beta\right)$ and $\text{pen}_\lambda\left(\beta\right)$ correspond to the negative logarithms of the likelihood and prior distribution, respectively.  Regularization leads to an optimization problem of the form 
\begin{equation}
\label{eqn:reg}
\begin{aligned}
& \underset{\beta \in \mathbb{R}^p}{\text{minimize}}
& & l\left(\beta\right) + \text{pen}_\lambda\left(\beta\right) \; . 
\end{aligned}
\end{equation}
Taking a probabilistic approach leads to a Bayesian hierarchical model
$$
p(y \mid \beta) \propto \exp\{-l(\beta)\} \; , \quad p(\beta) \propto \exp\{ -\text{pen}_\lambda\left(\beta\right) \} \ .
$$
The solution to the minimization problem estimated by regularization corresponds to the posterior mode, 
$ \hat{\beta} = {\rm arg \; max}_\beta \; p( \beta|y) $, where $ p(\beta|y)$ denotes the posterior distribution. 
For example, regression with a least squares log-likelihood 
subject to a penalty such as an $l_2$-norm (ridge) Gaussian probability model or $l_1$-norm (lasso) double exponential probability model.

The rest of the paper is outlined as follows.  Section \ref{s&s} defines the Bayesian $l_0$ regularization problem and explores the connections with spike-and-slab priors.  Section \ref{phi} introduces a novel connection between regularization and Bayesian inference.  Section \ref{survey} surveys recent developments on $l_0$ regularization and best subset selection problems.  Section \ref{sbr} describes the SBR algorithm for $l_0$ regularization.  Section \ref{eg} provides a comparison between SBR and other variable selection methods and models including Lasso and elastic net \citep{tibshirani1996,zou2005}, Bayesian bridge \citep{polson2014}, and spike-and-slab \citep{george1993}.  Finally, Section \ref{dis} concludes with directions for future research.

\section{Bayesian $l_0$ regularization \label{s&s}}

Consider a standard Gaussian linear regression model, where $X = [X_1, \ldots, X_p]\in \mathbb{R}^{n \times p}$ is a design matrix, $\beta = (\beta_1, \ldots, \beta_p)^T\in\mathbb{R}^p$ is a $p$-dimensional coefficient vector, and $e$ is an $n$-dimensional independent Gaussian noise.  After centralizing $y$ and all columns of $X$, we ignore the intercept term in the design matrix $X$ as well as $\beta$, and we can write
\begin{equation}
\label{eqn:linreg}
y = X\beta + e \ , \ \  \text{where } e \sim N(0, \sigma_e^2I_n) \ .
\end{equation}

To specify a prior distribution $p\left(\beta\right)$, we impose a sparsity assumption on $\beta$, where only a small portion of all $\beta_i$'s are non-zero.  In other words, $\|\beta\|_0 = k \ll p$, where $\|\beta\|_0 \defeq \#\{i : \beta_i\neq0\}$, the cardinality of the support of $\beta$, also known as the $l_0$ pseudo-norm of $\beta$.  A multivariate Gaussian prior ($l_2$ norm) leads to poor sparsity properties in this situation \citep[see, e.g.,][]{polson2010shrink}.

Sparsity-inducing prior distributions for $\beta$ can be constructed to impose sparsity.  The gold standard is a spike-and-slab priors \citep{jeffreys1961,mitchell1988,george1993}. Under these assumptions, each $\beta_i$ exchangeably follows a mixture prior consisting of $\delta_0$, a point mass at $0$, and a Gaussian distribution centered at zero. Hence we write,

\begin{equation}
\label{eqn:ss}
\beta_i | \theta, \sigma_\beta^2 \sim (1-\theta)\delta_0 + \theta N\left(0, \sigma_\beta^2\right) \ .
\end{equation}
Here $\theta\in \left(0, 1\right)$ controls the overall sparsity in $\beta$ and $\sigma_\beta^2$ accommodates non-zero signals.  This family is termed as the Bernoulli-Gaussian mixture model in the signal processing community.

A useful re-parameterization, the parameters $\beta$ is given by two independent random variable vectors $\gamma = \left(\gamma_1, \ldots, \gamma_p\right)'$ and $\alpha = \left(\alpha_1, \ldots, \alpha_p\right)'$ such that $\beta_i  =  \gamma_i\alpha_i$, with probabilistic structure
\begin{equation}
\label{eq:bg}
\begin{array}{rcl}
\gamma_i|\theta & \sim & \text{Bernoulli}(\theta) \ ;
\\
\alpha_i | \sigma_\beta^2 &\sim & N\left(0, \sigma_\beta^2\right) \ .
\\
\end{array}
\end{equation}
Since $\gamma_i$ and $\alpha_i$ are independent, the joint prior density becomes
$$
p\left(\gamma_i, \alpha_i \mid \theta, \sigma_\beta^2\right) =
\theta^{\gamma_i}\left(1-\theta\right)^{1-\gamma_i}\frac{1}{\sqrt{2\pi}\sigma_\beta}\exp\left\{-\frac{\alpha_i^2}{2\sigma_\beta^2}\right\}
\ , \ \ \ \text{for } 1\leq i\leq p \ .
$$
The indicator $\gamma_i\in \{0, 1\}$ can be viewed as a dummy variable to indicate whether $\beta_i$ is included in the model. \citep{soussen2011}

Let $S = \{i: \gamma_i = 1\} \subseteq \{1, \ldots, p\}$ be the ``active set" of $\gamma$, and $\|\gamma\|_0 = \sum\limits_{i = 1}^p\gamma_i$ be its cardinality.  The joint prior on the vector $\{\gamma, \alpha\}$ then factorizes as
$$
\begin{array}{rcl}
p\left(\gamma, \alpha \mid \theta, \sigma_\beta^2\right) & = & \prod\limits_{i = 1}^p p\left(\alpha_i, \gamma_i \mid \theta, \sigma_\beta^2\right) \\
& = & 
\theta^{\|\gamma\|_0}
\left(1-\theta\right)^{p - \|\gamma\|_0}
\left(2\pi\sigma_\beta^2\right)^{-\frac p2}\exp\left\{-\frac1{2\sigma_\beta^2}\sum\limits_{i = 1}^p\alpha_i^2\right\} \ .
\end{array}
$$

Let $X_\gamma \defeq \left[X_i\right]_{i \in S}$ be the set of ``active explanatory variables" and $\alpha_\gamma \defeq \left(\alpha_i\right)'_{i \in S}$ be their corresponding coefficients.  We can write $X\beta = X_\gamma \alpha_\gamma$.  The likelihood can be expressed in terms of $\gamma$, $\alpha$ as
$$
p\left(y \mid \gamma, \alpha, \theta, \sigma_e^2\right)
=
\left(2\pi\sigma_e^2\right)^{-\frac n2}
\exp\left\{
-\frac1{2\sigma_e^2}\left\|y - X_\gamma \alpha_\gamma\right\|_2^2
\right\} \ .
$$

Under this re-parameterization by $\left\{\gamma, \alpha\right\}$, the posterior is given by

$$
\begin{array}{rcl}
p\left(\gamma, \alpha \mid \theta, \sigma_\beta^2, \sigma_e^2, y\right) & \propto &
p\left(\gamma, \alpha \mid \theta, \sigma_\beta^2\right)
p\left(y \mid \gamma, \alpha, \theta, \sigma_e^2\right)\\
& \propto &
\exp\left\{-\frac1{2\sigma_e^2}\left\|y - X_\gamma \alpha_\gamma\right\|_2^2
-\frac1{2\sigma_\beta^2}\left\|\alpha\right\|_2^2
-\log\left(\frac{1-\theta}{\theta}\right)
\left\|\gamma\right\|_0
\right\} \ .
\end{array}
$$
Our goal then is to find the regularized maximum a posterior (MAP) estimator $$\arg\max\limits_{\gamma, \alpha}p\left(\gamma, \alpha \mid \theta, \sigma_\beta^2, \sigma_e^2, y \right) \ .$$By construction, the $\gamma$ $\in\left\{0, 1\right\}^p$ will directly perform variable selection.  Spike-and-slab priors, on the other hand, will sample the full posterior and calculate the posterior probability of variable inclusion.

Finding the MAP estimator is equivalent to minimizing over $\left\{\gamma, \alpha\right\}$ the regularized least squares objective function \citep{soussen2011}.

\begin{equation}
\label{obj:map}
\min\limits_{\gamma, \alpha}\left\|y - X_\gamma \alpha_\gamma\right\|_2^2
+ \frac{\sigma_e^2}{\sigma_\beta^2}\left\|\alpha\right\|_2^2
+ 2\sigma_e^2\log\left(\frac{1-\theta}{\theta}\right)
\left\|\gamma\right\|_0 \ .
\end{equation}
This objective possesses several interesting properties:
\begin{enumerate}
\item The first term is essentially the least squares loss function. 
\item The second term looks like a ridge regression penalty and has connection with the signal-to-noise ratio (SNR) $\sigma_\beta^2/\sigma_e^2$.  Smaller SNR will be more likely to shrink the estimate of $\alpha$ towards $0$.  If $\sigma_\beta^2 \gg \sigma_e^2$, the prior uncertainty on the size of non-zero coefficients is much larger than the noise level, that is, the SNR is sufficiently large, this term can be ignored.  This is a common assumption in spike-and-slab framework in that people usually want $\sigma_\beta \to \infty$ or to be ``sufficiently large" in order to avoid imposing harsh shrinkage to non-zero signals.
\item If we further assume that $\theta < \frac12$, meaning that the coefficients are known to be sparse \textit{a priori}, then $\log\left(\left(1-\theta\right) / \theta\right) > 0$, and the third term can be seen as an $l_0$ regularization.
\end{enumerate}
Therefore, our Bayesian objective inference is connected to $l_0$-regularized least squares, which we summarize in the following proposition.

\begin{proposition} (Spike-and-slab MAP \& $l_0$ regularization)

For some $\lambda > 0$, assuming $\theta < \frac12$, $\sigma_\beta^2 \gg \sigma_e^2$, the Bayesian MAP estimate defined by (\ref{obj:map}) is equivalent to the $l_0$ regularized least squares objective, for some $\lambda > 0$,

\begin{equation}
\label{obj:l0}
\min\limits_{\beta}
\frac12\left\|y - X\beta\right\|_2^2
+ \lambda
\left\|\beta\right\|_0 \ .
\end{equation}

\end{proposition}

\begin{proof}  First, assuming that
$$
\theta < \frac12, \ \ \  \sigma_\beta^2 \gg \sigma_e^2, \ \ \  \frac{\sigma_e^2}{\sigma_\beta^2}\left\|\alpha\right\|_2^2 \to 0 \ ,
$$
gives us an objective function of the form
\begin{equation}
\min\limits_{\gamma, \alpha}
\label{obj:vs}
\frac12 \left\|y - X_\gamma \alpha_\gamma\right\|_2^2
+ \lambda
\left\|\gamma\right\|_0,  \ \ \ \  \text{where } \lambda \defeq \sigma_e^2\log\left(\left(1-\theta\right) / \theta\right) > 0 \ .
\end{equation}

Equation (\ref{obj:vs}) can be seen as a variable selection version of equation (\ref{obj:l0}).  The interesting fact is that (\ref{obj:l0}) and (\ref{obj:vs}) are equivalent.  To show this, we need only to check that the optimal solution to (\ref{obj:l0}) corresponds to a feasible solution to (\ref{obj:vs}) and vice versa.  This is explained as follows.

On the one hand, assuming $\hat\beta$ is an optimal solution to (\ref{obj:l0}), then we can correspondingly define $\hat\gamma_i \defeq I\left\{\hat\beta_i \neq 0\right\}$, $\hat\alpha_i \defeq \hat\beta_i$, such that $\left\{\hat\gamma, \hat\alpha\right\}$ is feasible to (\ref{obj:vs}) and gives the same objective value as $\hat\beta$ gives (\ref{obj:l0}).

On the other hand, assuming $\left\{\hat\gamma, \hat\alpha\right\}$ is optimal to (\ref{obj:vs}), implies that we must have all of the elements in $\hat\alpha_\gamma$ should be non-zero, otherwise a new $\tilde\gamma_i \defeq I\left\{\hat\alpha_i \neq 0\right\}$ will give a lower objective value of (\ref{obj:vs}).  As a result, if we define $\hat\beta_i \defeq \hat\gamma_i\hat\alpha_i$, $\hat\beta$ will be feasible to (\ref{obj:l0}) and gives the same objective value as $\left\{\hat\gamma, \hat\alpha\right\}$ gives (\ref{obj:vs}).

Combining both arguments shows that the two problems (\ref{obj:l0}) and (\ref{obj:vs}) are equivalent.  Hence we can use results from non-convex optimization literature to find Bayes MAP estimators.
\end{proof}

\section{Bayesian regularization and Proximal Updating \label{phi}}

Section \ref{s&s} provides a connection between spike-and-slab and $l_0$ regularization, in the sense that $l_0$ regularization can be viewed as the MAP estimator of Bayesian spike-and-slab regression.  From a Bayesian perspective, the posterior mean is preferred to the MAP estimator due to its superior risk minimization properties under mean squared error.  \cite{starck2013} also discusses the inadequacy of interpreting sparse regularization as Bayesian MAP estimator.  We now show that the posterior mean estimation is also connected to regularization and introduce the connection with the help of proximal operators and Tweedie's formula \citep{efron2011}.

\cite{gribonval2011} considers the relationship between the MAP estimator and the posterior mean. \cite{polson2016} considers the relationship between posterior modes and envelopes. These approaches try to uncover the implicit prior that maps a posterior mean to a mode. As the posterior mean is designed to minimize the mean squared error and Bayes risk, this is the appropriate calculation.

\subsection{Posterior mean optimality}

Consider a normal mean problem in Bayesian setting,

\begin{equation}
\label{normmean}
\begin{array}{rcl}
y | \beta & \sim & N(\beta, \sigma_e^2) \ ;\\
\beta & \sim & p\left(\beta\right) \ .
\end{array}
\end{equation}
Then the optimal estimator of $\beta$ with respect to the quadratic loss is the posterior mean.  To calculate this posterior mean, $\hat\beta$, \cite{efron2011} introduced Tweedie's formula which, for the normal mean problem, gives

\begin{equation}
\label{eq:tweedie}
\hat{\beta} = E\left[ \beta \mid y \right] = y + \sigma_e^2\frac{d}{dy} \log m(y) \; ,
\end{equation}
where the marginal density of $y$ is
$$
m(y) = \int f\left( y \mid \beta\right) p( \beta ) d \beta \ .
$$
Here $f\left(y\mid\beta\right) = \frac{1}{\sqrt{2\pi}\sigma_e}\exp\left\{-\frac{(y - \beta)^2}{2\sigma_e^2}\right\}$ is the probability density function (pdf) of $N(\beta, \sigma_e^2)$.  The interpretation of the Bayesian correction term,
$$
\sigma_e^2\frac{d}{dy} \log m(y) \ ,
$$
is to ``regularize" the unbiased maximum likelihood estimator $y$, which provides the optimal bias-variance trade-off for prediction, see \cite{pericchi1992} for further discussion.  When both $m$ and $\sigma_e^2$ are unknown, \cite{donoho2013} proposes a procedure to estimate each term, which achieves the optimal Bayes risk, resulting in the plug-in estimator
$$
\hat\beta = y + \hat\sigma_e^2\frac{d}{dy} \log \hat m(y) \ .
$$

Another useful result applies Stein's risk function to Tweedie's formula.  Then we can derive this optimal Bayes risk \citep{robbins1956},
$$
R(\hat\beta) = \sigma_e^2\left(1 - \sigma_e^2 I\left(m\right)\right) \; ,
$$
where $I(m) = E_y\left[\left(\frac{d}{dy}\log m(y)\right)^2\right]$ is the Fisher Information for $m$.  This risk is optimal given the use of the posterior mean estimator.

Tweedie's formula can be generalized to Gaussian linear regression in the following theorem.

\begin{thm} Suppose in the Gaussian linear regression model,
$$
y = X\beta + e, \ \ \ \text{where } e\sim N\left(0, \Sigma\right) \ .
$$
Let $p\left(\beta\right)$ denote the prior density of $\beta$, and $m\left(y\right) = \int_\beta p\left(y \mid \beta\right)p\left(\beta\right)d\beta$ the marginal (prior predictive) density of $y$.  Then, assuming $\left(X^T\Sigma^{-1}X\right)^{-1}$ exists, the posterior mean of $\beta$ given $y$ is
\begin{equation}
\label{eq:lspmbeta}
E\left[\beta\mid y\right]
=
\left(X^T\Sigma^{-1}X\right)^{-1}X^T\left(\Sigma^{-1}y + \nabla_y\log m\left(y\right)\right) \ .
\end{equation}
\end{thm}

\begin{proof}
Let $N\left(y; X\beta, \Sigma\right)$ denote the multivariate normal density of $y |\beta\sim N\left(X\beta, \Sigma\right)$, then the posterior density of $\beta$ given $y$,
$$
p\left(\beta \mid y\right) = \displaystyle\frac{p\left(y\mid \beta\right)p\left(\beta\right)}{m\left(y\right)} = \frac{1}{m\left(y\right)}N\left(y; X\beta, \Sigma\right)p\left(\beta\right) \ .
$$
Therefore, the posterior mean of the quantity $\Sigma^{-1}\left(y - X\beta\right)$,
\begin{equation}
\label{eq:lspm}
\begin{array}{rcl}
E\left[\Sigma^{-1}\left(y - X\beta\right)\mid y\right]
&=&\int_\beta
\Sigma^{-1}\left(y - X\beta\right)
p\left(\beta\mid y\right)d\beta\\
&=&
\frac{1}{m\left(y\right)}
\int_\beta
\Sigma^{-1}\left(y - X\beta\right)
N\left(y; X\beta, \Sigma\right)p\left(\beta\right)
d\beta \ .
\end{array}
\end{equation}

Note that by the property of the multivariate normal density,
$$
\Sigma^{-1}\left(y - X\beta\right)
N\left(y; X\beta, \Sigma\right)
=
-\nabla_y N\left(y; X\beta, \Sigma\right) \ ,
$$
and so (\ref{eq:lspm}) becomes
$$
E\left[\Sigma^{-1}\left(y - X\beta\right)\mid y\right]
=
\frac{1}{m\left(y\right)}
\int_\beta
-\nabla_y N\left(y; X\beta, \Sigma\right)
p\left(\beta\right)
d\beta 
=
-\nabla_y \log m\left(y\right) \ .
$$

Multiplying both sides by $X^T$ and assuming $\left(X^T\Sigma^{-1}X\right)^{-1}$ exists, the posterior mean of $\beta$ given $y$ becomes
\vspace{0.5pc}

\hfill
$\displaystyle E\left[\beta\mid y\right] = \left(X^T\Sigma^{-1}X\right)^{-1}X^T\left(\Sigma^{-1}y + \nabla_y\log m\left(y\right)\right) \ .$
\hfill
\end{proof}

It's easy to see that, similar to Tweedie's formula for the normal means problem, the posterior mean in the Gaussian linear regression (\ref{eq:lspmbeta}) consists of two parts.  One is the usual weighted least squares solution, and the other is a Bayesian correction by the gradient of the prior predictive score, $\nabla_y \log m\left(y\right)$.  \cite{griffin2010} gives the equivalent result in the form when the least squares estimator $\hat\beta$ instead of $y$ is conditioned on. \cite{masreliez1975} discusses the posterior mean under Gaussian prior but non-Gaussian likelihood. \citep{pericchi1992}

\subsection{Regularized linear regression}

Proximal operators and Tweedie's formula also provide a way to connect the Bayesian posterior mean and the regularized least squares.  Specifically, we want to find a $\phi$, such that the Bayesian posterior estimator $\hat{\beta} = E\left[ \beta \mid y \right]$ with the prior $p$ is the same as the solution to the $\phi$-regularized least squares.  That is,

\begin{equation}
\label{obj:pmreg}
\hat{\beta} = E\left[ \beta \mid y \right] = \arg\min_\beta \left\{\frac1{2\sigma_e^2}(y -\beta)^2 + \phi(\beta)\right\}.
\end{equation}

We now use the theory of proximal mappings to re-write this estimator.  First, here are some definitions.  The Moreau envelope $E_{\gamma f} (x)$ and proximal mapping $\prox_{\gamma f} (x)$ of a convex function are defined as
\begin{equation}
\label{def:prox}
\begin{array}{rcl}
E_{\gamma f} (x) &=& \inf_{z } \left\{f(z) + \frac{1}{2\gamma} \enorm{z - x}^2  \right\}  \leq f(x) \ ;\\
\prox_{\gamma f} (x) &=& \arg \min_{z } \left\{  f(z)+ \frac{1}{2\gamma} \enorm{z - x}^2  \right\} \, .
\end{array}
\end{equation}
The Moreau envelope is a regularized version of $f$ and approximates $f$ from below, and has the same set of minimizing values as $f$.  
The proximal mapping returns the value that solves the minimization problem defined by the Moreau envelope.  It balances two goals: minimizing $f$, and staying near $x$. 

Now, observe that if $\hat{z}(x) = \prox_{\gamma f}(x)$ is the value that achieves the minimum,
$$
\nabla \left\{f(\hat z) + \frac{1}{2\gamma} \enorm{\hat z - x}^2  \right\}
=\nabla f(\hat z) + \frac1\gamma(\hat z - x) = 0 \, ,
$$
which leads to $\hat z = x - \gamma\nabla f(\hat z) \, .$ By construction of the envelope,
$$
\nabla E_{\gamma f}(x) = \nabla  \inf_{z } \left\{f(z) + \frac{1}{2\gamma} \enorm{z - x}^2  \right\} = \frac{1}{\gamma}[x - \hat{z}(x)] \, ,
$$
This leads to the fundamental proximal relation $\hat z = x - \gamma \nabla E_{\gamma f}(x) \, .$
Therefore, write
\begin{equation}
\label{eq:prox}
\prox_{\gamma f}(x) = x - \gamma\nabla f\left[\prox_{\gamma f}(x)\right] = x - \gamma \nabla E_{\gamma f}(x) \, .
\end{equation}

Meanwhile, the definition of proximal mapping (\ref{def:prox}) and its property (\ref{eq:prox}) give us

$$
\hat{\beta} = \arg\min_\beta \left\{\frac1{2\sigma_e^2}(y -\beta)^2 + \phi(\beta)\right\} = \prox_{\sigma_e^2\phi}\left(y\right) = y - \sigma_e^2\nabla \phi\left(\hat\beta\right) \ .
$$

Combining with Tweedie's formula (\ref{eq:tweedie}), gives

\begin{equation}
\label{eq:pmreg}
\nabla \phi\left(\hat\beta\right) = - \frac{d}{dy} \log m(y) \ .
\end{equation}

Hence, if we want to match a regularized least squares with a posterior mean, we can ``solve'' for the penalty $\phi$, given a marginal distribution $m(y)$, via the equation for the proximal mapping (\ref{eq:pmreg}).  If $ \phi $ is non-differentiable at a point, we replace $ \nabla $ by $ \partial $.

\cite{gribonval2011} provides the following answer.  Given any $z$, find $\hat y$ such that $E\left[ \beta \mid\hat y\right] = z$. Then the penalty
\begin{equation}
\label{eq:phi}
\phi(z) = -\frac1{2\sigma_e^2}(\hat y - z)^2 - \log m\left( \hat y\right) + c \; ,
\end{equation}
with the constant $c$ to ensure that $\phi(0) = 0$. To see why this construction makes sense, simply take derivatives with respect to $\hat y$ on both sides, and get
$$
\begin{array}{rcl}
- \frac{d}{d\hat y} \log m(\hat y) & = & \nabla\phi(z)\frac{dz}{d\hat y} + \frac1{\sigma_e^2}\left(\hat y - z\right)\left(1 - \frac{dz}{d\hat y}\right) \\
\left[\text{ by (\ref{eq:tweedie}) }\right] & = & \nabla\phi(z)\frac{dz}{d\hat y} + \frac1{\sigma_e^2}\left(-\sigma_e^2\frac{d}{d\hat y} \log m(\hat y)\right)\left(1 - \frac{dz}{d\hat y}\right) \\
& = & - \frac{d}{d\hat y} \log m(\hat y) + \left(\nabla\phi(z) + \frac{d}{d\hat y} \log m(\hat y)\right)\frac{dz}{d\hat y} \\
\left[\text{ by (\ref{eq:pmreg}) }\right]& = & - \frac{d}{d\hat y} \log m(\hat y) \ .
\end{array}
$$

To summarize, the solution to the regularized least squares problem
$$
\min_\beta\frac1{2\sigma_e^2}\left(y - \beta\right)^2 + \phi\left(\beta\right)
$$
is the posterior mode with the prior $p\left(\beta\right) \propto \exp\left(-\phi\left(\beta\right)\right)$.  Our proximal operators and Tweedie's formula discussion shows that the regularized least squares solution can also be viewed as the posterior mean under an implied prior $p\left(\beta\right)$, see \cite{strawderman2013}.

To illustrate our result, When $p$ is sparsity-inducing, such as the spike-and-slab, we can construct the associated penalty $\phi$, which is typically non-convex for both Gaussian and Laplace cases.

\subsection{Example: Spike-and-slab Gaussian \& Laplace prior}

For the normal mean problem (\ref{normmean}), assuming $p\left(\beta\right)$ is the aforementioned spike-and-slab (Bernoulli Gaussian) prior (\ref{eqn:ss}), the marginal distribution of $y$ is a mixture of two mean zero normals,
$$
\left.y\mid\theta\right. \sim \left(1 - \theta\right) N\left(0, \sigma_e^2\right) + \theta N\left(0, \sigma_e^2 + \sigma_\beta^2\right) \ .
$$
The posterior mean $E\left[\beta | y\right]$ is given by
$$
\hat\beta^{BG} = w(y) y \ , \; \; {\rm where} \; \; w(y) =   \frac{\sigma_\beta^2 }{\sigma_e^2 + \sigma_\beta^2} 
\left(1 + 
\frac{
\left(1-\theta\right) \frac{1}{\sqrt{2\pi}\sigma_e}\exp\left\{-\frac{y^2}{2\sigma_e^2}\right\}
}
{
\theta \frac{1}{\sqrt{2\pi}\sigma_e}\exp\left\{-\frac{y^2}{2\sigma_e^2 + \sigma_\beta^2}\right\}
}
\right)^{-1} \ .
$$
Thus, $\forall z\in \mathbb{R}$, we can find $\hat y$ such that $w\left(\hat y\right)\hat y = z$; then the penalty $\phi^{BG}$ associated with the Bayesian posterior mean with the Bernoulli-Gaussian prior can be obtained by (\ref{eq:phi}).  $\phi^{BG}$ doesn't have an analytical form, but can be computed numerically.

\cite{amini2012} argued that Bernoulli-Gaussian priors are usually not applicable to real-world signals, and proposed Bernoulli-Laplace priors which are infinitely divisible and more appropriate for sparse signal processing.  The Bernoulli-Laplace priors are very similar to the Bernoulli-Gaussian ones, their only difference being that the ``slab" parts are replaced by Laplace distributions.  Therefore, the prior

\begin{equation}
\label{prior:bl}
p\left(\beta\mid\sigma_\beta\right) = \left(1-\theta\right)\delta_0 + \theta \frac{1}{\sqrt{2}\sigma_\beta}\exp\left(-\frac{\sqrt{2}}{\sigma_\beta}\left|\beta\right|\right) \ .
\end{equation}

\cite{mitchell1994} and \cite{hans2009} studied the marginal and posterior distribution with the Laplace prior.  With their results, the marginal density of $y$ with the Bernoulli-Laplace prior (\ref{prior:bl}) is given by
$$
m\left(y\right) = \left(1-\theta\right) \frac{1}{\sqrt{2\pi}\sigma_e}\exp\left\{-\frac{y^2}{2\sigma_e^2}\right\} + \theta
\frac{1}{\sqrt{2}\sigma_\beta}
\exp\left\{\frac{\sigma_e^2}{\sigma_\beta^2}\right\}
\left(F_{\sigma_\beta}\left(y\right) + 
F_{\sigma_\beta}\left(-y\right)\right) \ ,
$$
where
$$
F_{\sigma_\beta}\left(y\right) = 
\exp\left\{\frac{\sqrt{2}y} {\sigma_\beta}\right\}
\Phi\left(
-\frac{y}{\sigma_e}
-\frac{\sqrt{2}\sigma_e}{\sigma_\beta}
\right) \ .
$$
Here $\Phi$ is the cumulative distribution function (cdf) of the standard normal.  The posterior mean $E\left[\beta | y\right]$ is then given by

$$
\hat\beta^{BL} = 
\left(y - \left[\frac{F_{\sigma_\beta}(-y) - F_{\sigma_\beta}(y)}{F_{\sigma_\beta}(-y) + F_{\sigma_\beta}(y)}\right]\frac{\sqrt{2}\sigma_e^2}{\sigma_\beta} \right)
\left(1 + 
\frac{
\left(1-\theta\right) 
\frac{1}{\sqrt{2\pi}\sigma_e}\exp\left\{-\frac{y^2}{2\sigma_e^2}\right\}
}
{
\theta
\frac{1}{\sqrt{2}\sigma_\beta}
\exp\left\{\frac{\sigma_e^2}{\sigma_\beta^2}\right\}
\left(F_{\sigma_\beta}(y) + F_{\sigma_\beta}(-y)\right)
}
\right)^{-1}
$$
Similarly, we are also able to find the penalty $\phi^{BL}$ associated with this Bernoulli-Laplace prior numerically by (\ref{eq:phi}). It's worth noting that this prior is a special case of the spike-and-slab Lasso prior proposed by \cite{rockova2016}. In that paper the authors use a mixture of two Laplace distributions, one of which is very close to $\delta_0$ as its variance goes to zero. Both priors are capable of striking a balance between hard-thresholding and soft-thresholding.

For comparison, $\hat\beta^{BG}$ and $\hat\beta^{BL}$ are plotted in Figure \ref{fig:pm}; $\phi^{BG}$ and $\phi^{BL}$ are plotted in Figure \ref{fig:phi}.  Both priors shrink small observations towards zero.  For large observations, when $\sigma_\beta$ is small, Bernoulli-Gaussian, like ridge regression, unnecessarily penalizes large observations too much, whereas Bernoulli-Laplace is more like Lasso.  As $\sigma_\beta$ gets larger, both priors get closer to hard-thresholding, and their associated penalties $\phi$ closer to SCAD-like non-convex penalties \citep{fan2001}.

\begin{figure}[!htb]
\centering
   \begin{minipage}{0.49\textwidth}
     \includegraphics[width=\linewidth]{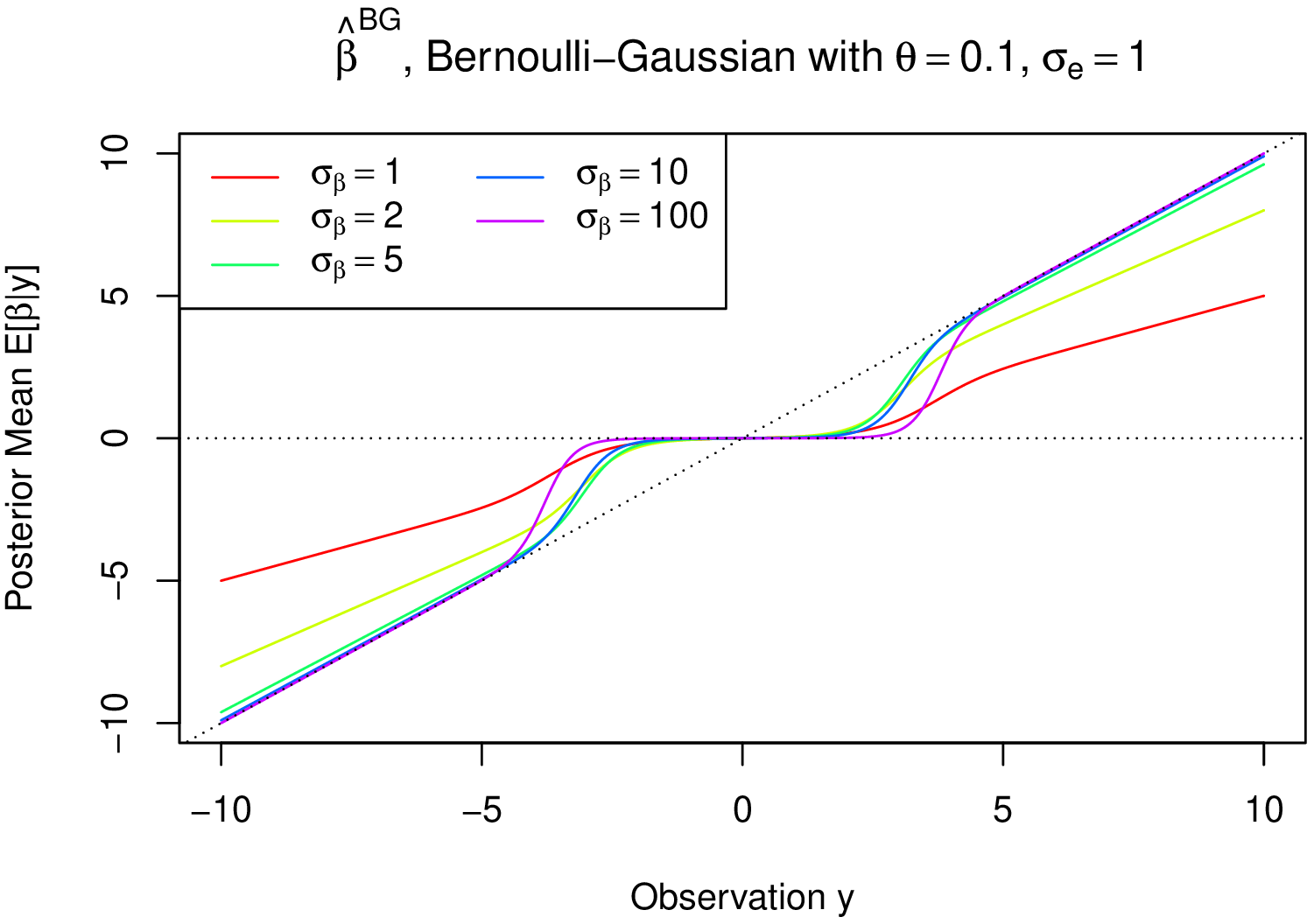}
   \end{minipage}
   \begin {minipage}{0.49\textwidth}
     \includegraphics[width=\linewidth]{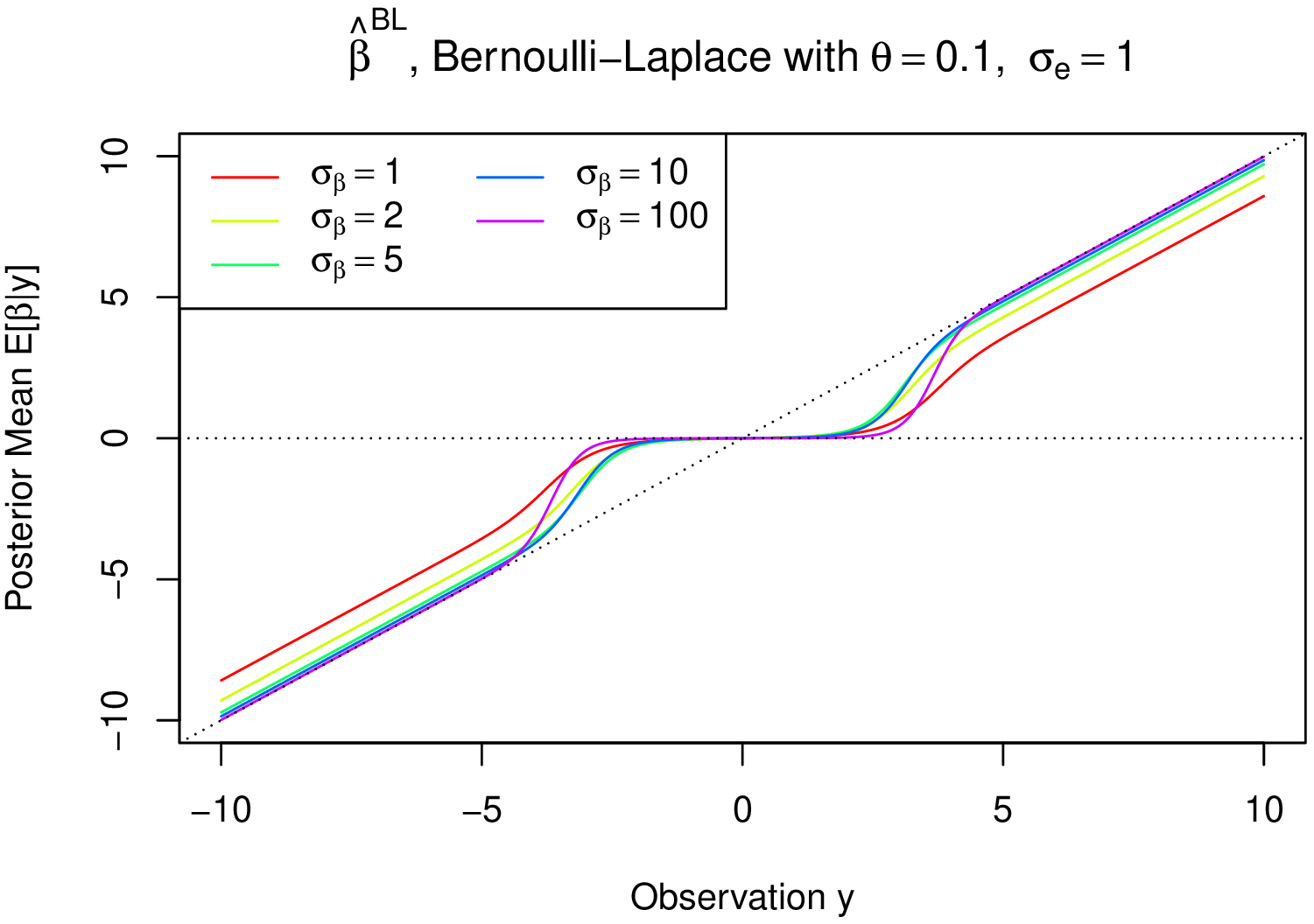}
   \end{minipage}
  \caption{Posterior mean $\hat\beta = E\left[\beta \mid y\right]$ with Bernoulli-Gaussian (left) and Bernoulli-Laplace (right) priors.  \textit{Both priors shrink small observations towards zero.  When $\sigma_\beta$ is small, Bernoulli-Gaussian priors shrink large observations more heavily than Bernoulli-Laplace priors which are more like soft-thresholding.  As $\sigma_\beta$ gets larger, both get closer to hard-thresholding.}}
\label{fig:pm}
\end{figure}

\begin{figure}[!htb]
\centering
   \begin{minipage}{0.49\textwidth}
     \includegraphics[width=\linewidth]{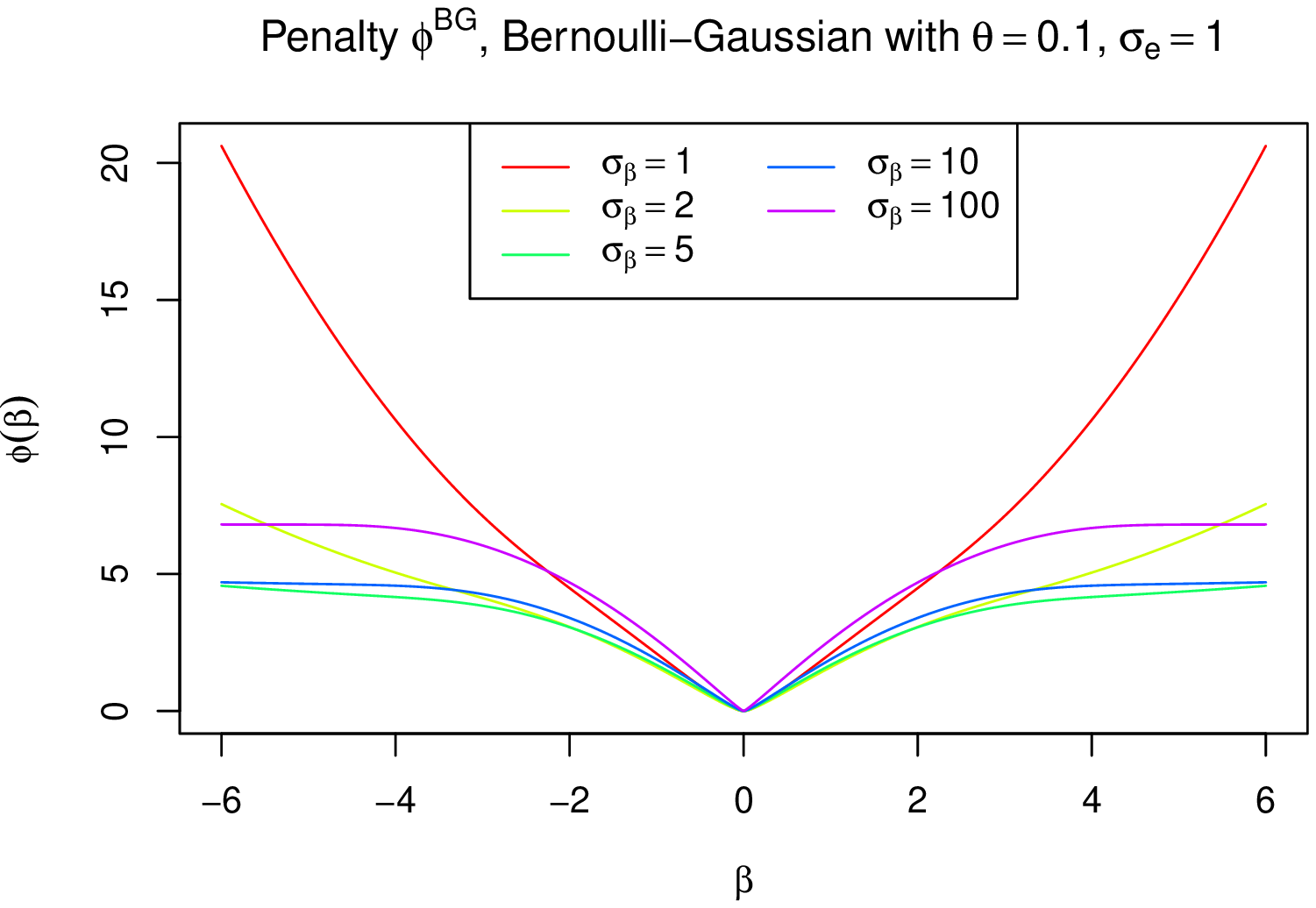}
   \end{minipage}
   \begin {minipage}{0.49\textwidth}
     \includegraphics[width=\linewidth]{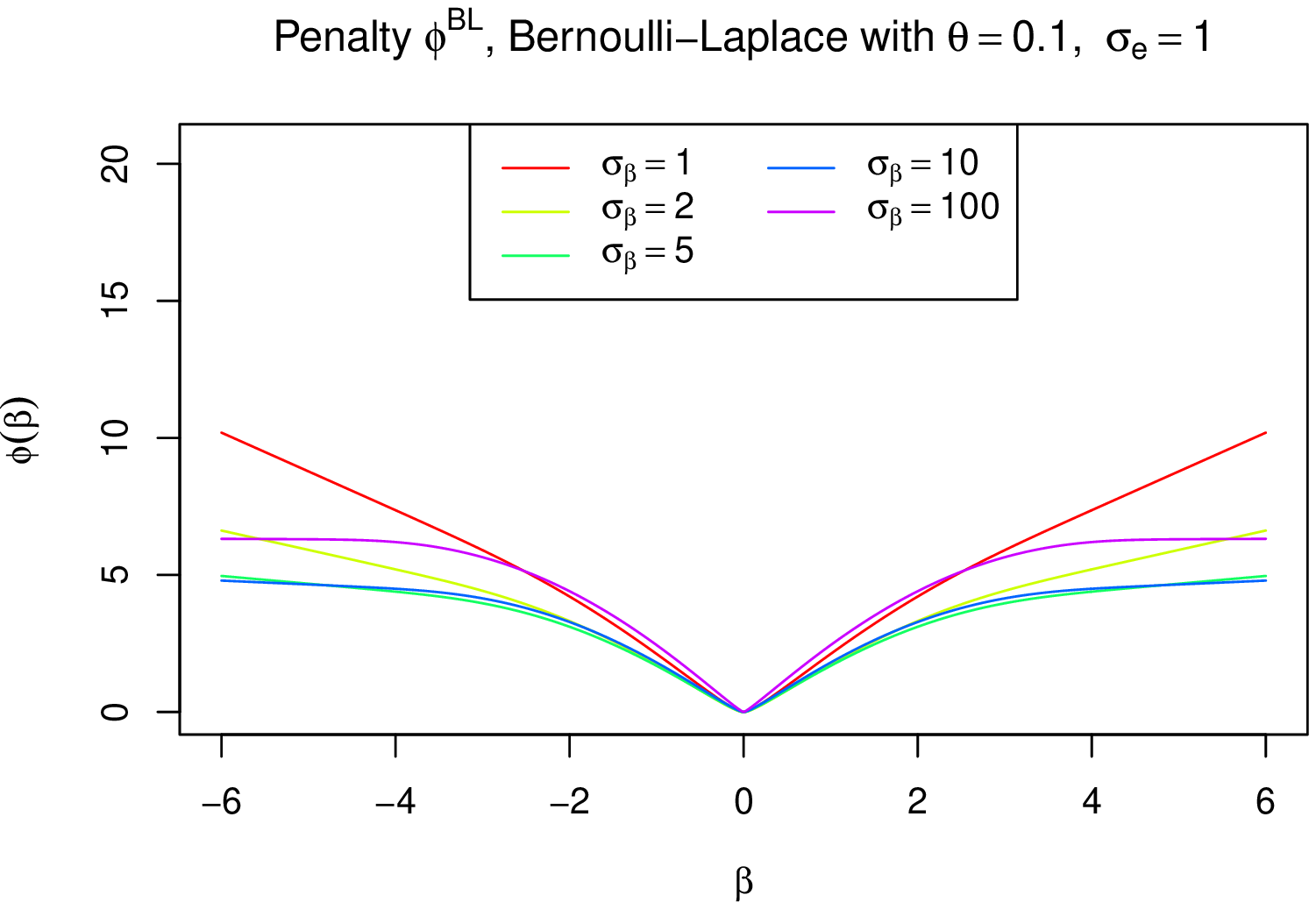}
   \end{minipage}
  \caption{Penalty $\phi$ associated with the posterior mean of Bernoulli-Gaussian (left) and Bernoulli-Laplace (right) priors.  \textit{Both $\phi^{BG}$ and $\phi^{BL}$ look ``spiky" around zero, seemingly to induce sparsity for small observations, although they are actually differentiable everywhere.  When $\sigma_\beta$ is small, the penalties associated with Bernoulli-Gaussian priors behave like ridge regression for large observations, whereas those associated with Bernoulli-Laplace priors appear to have a Lasso flavor.  As $\sigma_\beta$ gets larger, both get closer to non-convex penalties such as SCAD.}}
\label{fig:phi}
\end{figure}

\section{Computing the $l_0$-regularized regression solution \label{survey}}

We now turn to the problem of computation.  $l_0$-regularized least squares (\ref{obj:l0}) is closely related to the best subset selection in linear regression as follows.

\begin{equation}
  \label{obj:subset}
  \begin{array}{rl}
\min\limits_{\beta} & \frac12\|y - X\beta\|_2^2\\
\text{s.t.} & \|\beta\|_0 \leq k \ .
  \end{array}
\end{equation}
The $l_0$-regularized least squares (\ref{obj:l0}) can be seen as (\ref{obj:subset})'s Lagrangian form.  However, due to high non-convexity of the $l_0$-norm, (\ref{obj:l0}) and (\ref{obj:subset}) are connected but not equivalent.  In particular, for any given $\lambda \geq 0$, there exists an integer $k \geq 0$, such that (\ref{obj:l0}) and (\ref{obj:subset}) have the same global minimizer $\hat\beta$.  However, it's not true the other way around.  It's possible, even common, that for a given $k$, we cannot find a $\lambda \geq 0$, such that the solutions to (\ref{obj:subset}) and (\ref{obj:l0}) are the same.

Indeed, for $k \in \left\{1, 2, \ldots, p\right\}$, let $\hat\beta_k$ be respective optimal solutions to (\ref{obj:subset}) and $f_k$ respective optimal objective values, and so $f_1 \geq f_2 
\geq \cdots \geq f_p$.  If we want a solution $\hat\beta_\lambda$ to (\ref{obj:l0}) has $\left\|\hat\beta_\lambda\right\|_0 = k$, we need to find a $\lambda$ such that

$$
\max\limits_{i > k}\left\{f_k - f_i\right\} \leq \lambda \leq \min\limits_{j < k}\left\{f_j - f_k\right\} \ ,
$$
with the caveat that such $\lambda$ needs not exist.

Both problems involve discrete optimization and have thus been seen as intractable for large-scale data sets.  As a result, in the past, $l_0$ norm is usually replaced by its convex relaxation $l_1$ norm to facilitate computation.  However, it's widely known that the solutions of $l_0$ norm problems provide superior variable selection and prediction performance compared with their $l_1$ convex relaxation such as Lasso.  \cite{zhang2014lower} studies the statistical properties of the theoretical solution to (\ref{obj:l0}), and points out that the solution to the $l_0$-regularized least squares should be better than Lasso in terms of variable selection especially when we have a design matrix $X$ that has high collinearity among its columns.

\cite{bertsimas2016} introduced a first-order algorithm to provide a stationary solution $\beta^*$ to a class of generalized $l_0$-constrained optimization problem, with convex $g$,
\begin{equation}
  \label{eq:gen}
  \begin{array}{rl}
\min\limits_{\beta} & g(\beta)\\
\text{s.t.} & \|\beta\|_0 \leq k \ .
  \end{array}
\end{equation}
Let $L$ be the Lipschitz constant for $\nabla g$ such that $\forall\beta_1, \beta_2$, $\|\nabla g(\beta_1) - \nabla g(\beta_2)\| \leq L \|\beta_1  - \beta_2\|$.  Their ``Algorithm 1" is as follows.

\begin{enumerate}
\item Initialize $\beta^0$ such that $\left\|\beta^0\right\|_0 \leq k$.
\item For $t \geq 0$, obtain $\beta^{t + 1}$ as
\begin{equation}
\label{subset:algo1}
\beta^{t + 1} = H_k\left(\beta^t - \frac1L\nabla g\left(\beta^t\right)\right) \ ,
\end{equation}
until convergence to $\beta^*$.
\end{enumerate}
where the operator $H_k\left(\cdot\right)$ is to keep the largest $k$ elements of a vector as the same, whilst to set all else to zero.  It can also be called the hard thresholding at the $k^\text{th}$ largest element.  In the least squares setting when $g(\beta) = \frac12\|y - X\beta\|_2^2$, $\nabla g$ and $L$ are easy to compute.  \cite{bertsimas2016} then uses the stationary solution $\beta^*$ obtained by the aforementioned algorithm (\ref{subset:algo1}) as a warm start for their mixed integer optimization (MIO) scheme to produce a ``provably optimal solution" to the best subset selection problem (\ref{obj:subset}).

It's worth pointing out that the key iteration step (\ref{subset:algo1}) is connected to the proximal gradient descent (PGD) algorithm many have used to solve the $l_0$-regularized least squares (\ref{obj:l0}), as well as other non-convex regularization problems.  PGD methods solve a general class of problems such as
\begin{equation}
  \label{obj:pgd}
  \begin{array}{rl}
\min\limits_{\beta} & g(\beta) + \lambda\phi(\beta) \ ,
  \end{array}
\end{equation}
where $g$ is the same as in (\ref{eq:gen}), and $\phi$, usually non-convex, is a regularization term.  In this framework, in order to obtain a stationary solution $\beta^*$, the key iteration step is
\begin{equation}
  \label{pgd:algo}
\beta^{t + 1} = \prox_{\lambda\phi}\left(\beta^t - \frac1L\nabla g\left(\beta^t\right)\right) \ ,
\end{equation}
where $\beta^t - \frac1L\nabla g(\beta^t)$ can be seen as a gradient descent step for $g$ and $\prox_{\lambda\phi}$ is the proximal operator for $\lambda\phi$.  In $l_0$-regularized least squares, $\lambda\phi\left(\cdot\right) = \lambda\left\|\cdot\right\|_0$, and its proximal operator $\prox_{\lambda\|\cdot\|_0}$ is just the hard thresholding at $\lambda$.  That is, $\prox_{\lambda\|\cdot\|_0}$ is to keep the same all elements no less than $\lambda$, whilst to set all else to zero.  As a result, the similarity between (\ref{subset:algo1}) and (\ref{pgd:algo}) are quite obvious.

In a recent work, \cite{jewell2017} proposes an exact algorithm to obtain the global minimum of $l_0$-regularized optimization in a computational neuroscience context. Consider the optimization problem

$$
\min\limits_{c_1,\ldots,c_n}\left\{\frac12\sum\limits_{i = 1}^n\left(y_i - c_i\right)^2 + \lambda\sum\limits_{i = 2}^{n}\mathbb{I}{\left(c_i - \gamma c_{i-1}\right)}\right\} \ .
$$

By exploiting the sequential time series nature of the problem, one can recast the problem as a changepoint detection problem and use the available results in that literature to design a dynamic programming algorithm. Instead of trying to simultaneously find all $i$'s where $c_i - \gamma c_{i - 1} \neq 0$, the algorithm aims to find them sequentially, from $i = 1$ to $i = n$ at each step. In this sense, it can reach a global minimum within $\mathbb{O}\left(n^2\right)$. The authors further speed up the algorithm by pruning the set of possible changepoints to at each step of the sequential search, and reduce the expected time cost to $\mathbb{O}\left(n\right)$.

\section{Single best replacement (SBR) algorithm \label{sbr}}

The single best replacement (SBR) algorithm, originally developed by \cite{soussen2011}, provides solution to the variable selection regularization (\ref{obj:vs}).  Since (\ref{obj:vs}) and the $l_0$-regularized least squares (\ref{obj:l0}) are equivalent, SBR also provides a practical way to give a sufficiently good local optimal solution to the NP-hard $l_0$ regularization.

Take a look at the objective (\ref{obj:vs}).  For any given variable selection indicator $\gamma$, we have an active set $S = \left\{i: \gamma_i = 1\right\}$, based on which the minimizer $\hat\alpha_\gamma$ of (\ref{obj:vs}) has a closed form.  $\hat\alpha_\gamma$ will set every coefficients outside $S$ to zero, and regress $y$ on $X_\gamma$, the variables inside $S$.  Therefore, the minimization of the objective function can be determined by $\gamma$ or $S$ alone.  Accordingly, the objective function (\ref{obj:vs}) can be rewritten as follows.

\begin{equation}
\label{obj:sbr}
\min\limits_{S}
f_{SBR}(S) = 
\frac12 \left\|y - X_S \beta_S \right\|_2^2
+ \lambda
\left|S\right| \ .
\end{equation}
The SBR algorithm thus tries to minimize $f_{SBR}(S)$ via choosing the optimal $\hat S$.

The algorithm works as follows.  Suppose we start as an initial $S$, usually the empty set.  At each iteration, SBR aimes to find a ``single change of $S$", that is, a single removal from or adding to $S$ of one element, such that this single change decreases $f_{SBR}(S)$ the most.  SBR stops when no such change is available, or in other words, any single change of $\gamma$ or $S$ will only give the same or larger objective value.  Therefore, intuitively SBR stops at a local optima of $f_{SBR}(S)$.

SBR is essentially a stepwise greedy variable selection algorithm.  At each iteration, both adding and removal are allowed, so this algorithm is one example of the ``forward-backward" stepwise procedures.  It's provable that with this feature the algorithm ``can escape from some [undesirable] local minimizers" of $f_{SBR}(S)$ \citep{soussen2015}.  Therefore, SBR can solve the $l_0$-regularized least squares in a sub-optimal way, providing a satisfactory balance between efficiency and accuracy.

We are now writing out the algorithm more formally.  For any currently chosen active set $S$, define a single replacement $S\cdot i, i\in\left\{1, \ldots, p\right\}$ as $S$ adding or removing a single element $i$,

$$
S\cdotp i \defeq
\begin{cases}
S\cup\{i\}, & i\notin S \\
S\backslash \{i\}, & i\in S 
\end{cases} \ .
$$
Then we compare the objective value at current $S$ with all of its single replacements $S\cdot i$, and choose the best one.  SBR proceeds as follows.

\begin{description}
\item Step 0: Initialize $S_0$.  Usually, $S_0 = \emptyset$.  Compute $f_{SBR}(S_0)$.  Set $k = 1$.
\item Step $k$: For every $i \in \{1, \ldots, p\}$, compute $f_{SBR}(S_{k-1} \cdot i)$.  Obtain the single best replacement $j \defeq \arg\min\limits_{i}f_{\text{SBR}}(S_{k-1} \cdot i)$.
\begin{enumerate}
\item If $f_{SBR}(S_{k-1} \cdot j) \geq f_{SBR}(S_{k-1})$, stop. Report $\hat S = S_{k - 1}$ as the solution.
\item Otherwise, set $S_{k} = S_{k-1} \cdot j$, $k = k+1$, and repeat step $k$.
\end{enumerate}
\end{description}

\cite{soussen2011} shows that SBR always stops within finite steps.  With the output $\hat S$, the locally optimal solution to the $l_0$-regularized least squares $\hat\beta$ is just the coefficients of $y$ regressed on $X_{\hat S}$ and zero elsewhere.

In order to include both forward and backward steps in the variable selection process, the algorithm needs to compute $f_{SBR}(S_{k-1} \cdot i)$ for every $i$ at every step.  Because it involves a one-column update of current design matrix $X_{S_{k - 1}}$, this computation can be made very efficient by using the Cholesky decomposition, without explicitly calculating $p$ linear regressions at each step \citep{soussen2011}.  An \texttt{R} package implementation of the algorithm is available upon request.

\section{Applications \label{eg}}

\subsection{Statistical properties of SBR and $l_0$ regularization \label{perform}}

The design matrix $X$ in this experimentation has $n = 120$ rows and $p = 100$ columns.  In order to impose high collinearity in the columns of $X$, we construct it in the following way.

\begin{enumerate}
\item Construct a $p \times d$ matrix $L$ consisting of $N\left(0, 1\right)$ random samples and obtain $\Sigma_X = LL^T + I_p$.  If $d \ll p$, $\Sigma_X$ will have a low-rank structure.  Here we use $d = 5$.
\item Sample each of the $n$ rows of $X$ from the multivariate normal distribution $N_p\left(0, \Sigma_X\right)$.
\item Centralize and normalize the columns of $X$ such that each column sums to zero and has unit $l_2$ norm.
\end{enumerate}
A design matrix $X$ constructed this way has highly collinear columns.  $\beta$ is a highly sparse coefficient vector with $100$ elements, $90$ of which are zero, and the rest $10$ randomly chosen to be $\left\{-5, -4, -3, -2, -1, 1, 2, 3, 4, 5\right\}$.  The noise vector $e$ is sampled from $N\left(0, \sigma_e^2\right)$.  In this setting, the signal-to-noise ratio was defined as 
$$
\text{SNR}=10\log_{10}\left[\frac{\sigma^2(X\beta)}{\sigma_e^2}\right] \ ,
$$
and $\sigma_e$ is determined such that $\text{SNR}=20\text{ dB}$.  Finally, let $y = X\beta + e$ be the observation.

Essentially all kinds of regularization methods, including ridge regression, Lasso, and $l_0$ regularization, share a common difficulty: to find a suitable regularization parameter $\lambda$.  A thorough theoretical treatment on the optimal $\lambda$ hasn't been established, and in practice cross validation is often used.

Meanwhile, one of the advantages of $l_0$ regularization is that, compared with its $l_1$ counterpart, the estimates of coefficients are relatively insensitive to $\lambda$ when it's large enough, as shown in Figure \ref{fig:solutionpath}.  Therefore, we don't have to worry too much about choosing a particularly optimal $\lambda$.

\begin{figure}[!htb]
\centering
   \begin{minipage}{0.49\textwidth}
     \includegraphics[width=\linewidth]{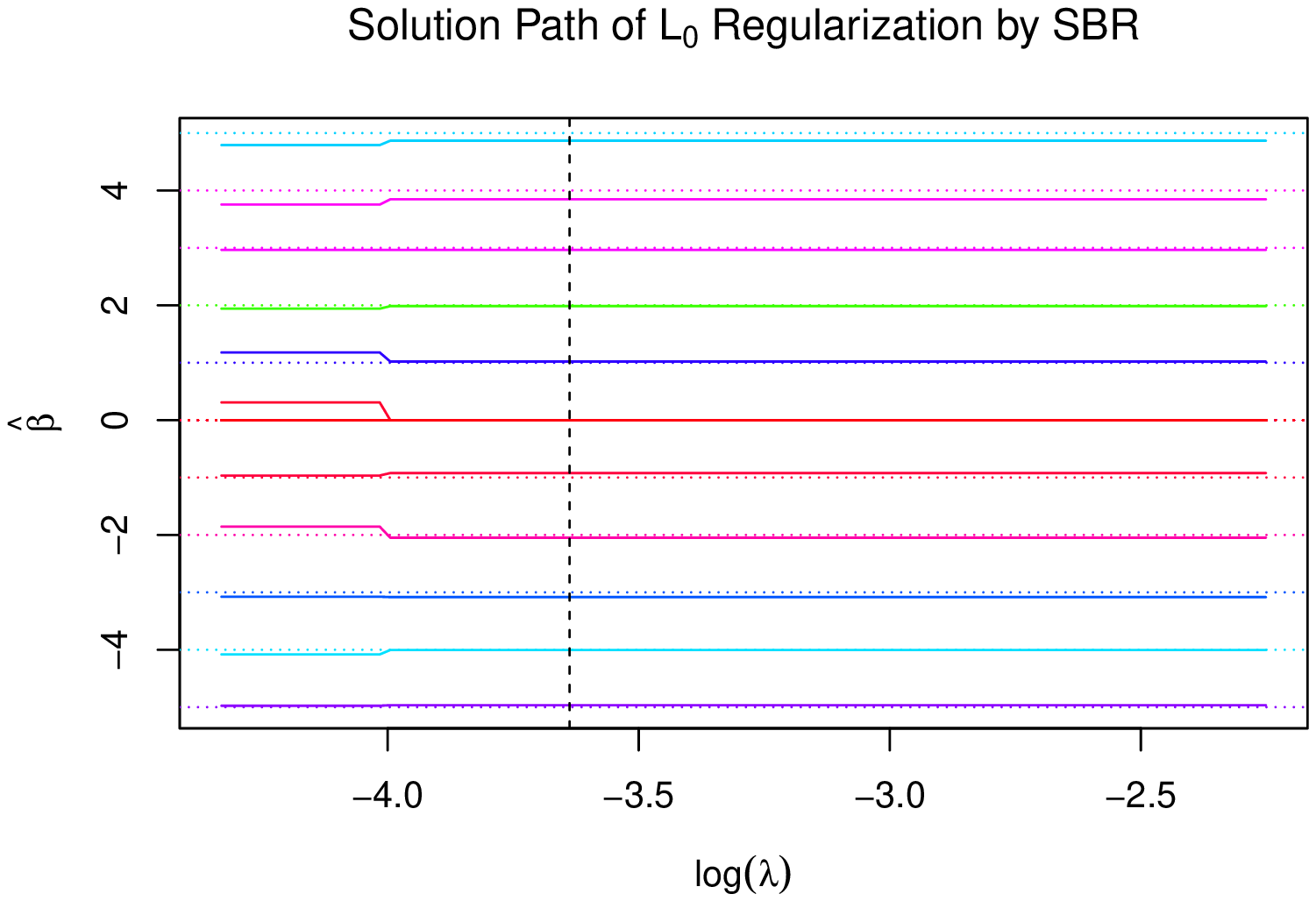}
   \end{minipage}
   \begin {minipage}{0.49\textwidth}
     \includegraphics[width=\linewidth]{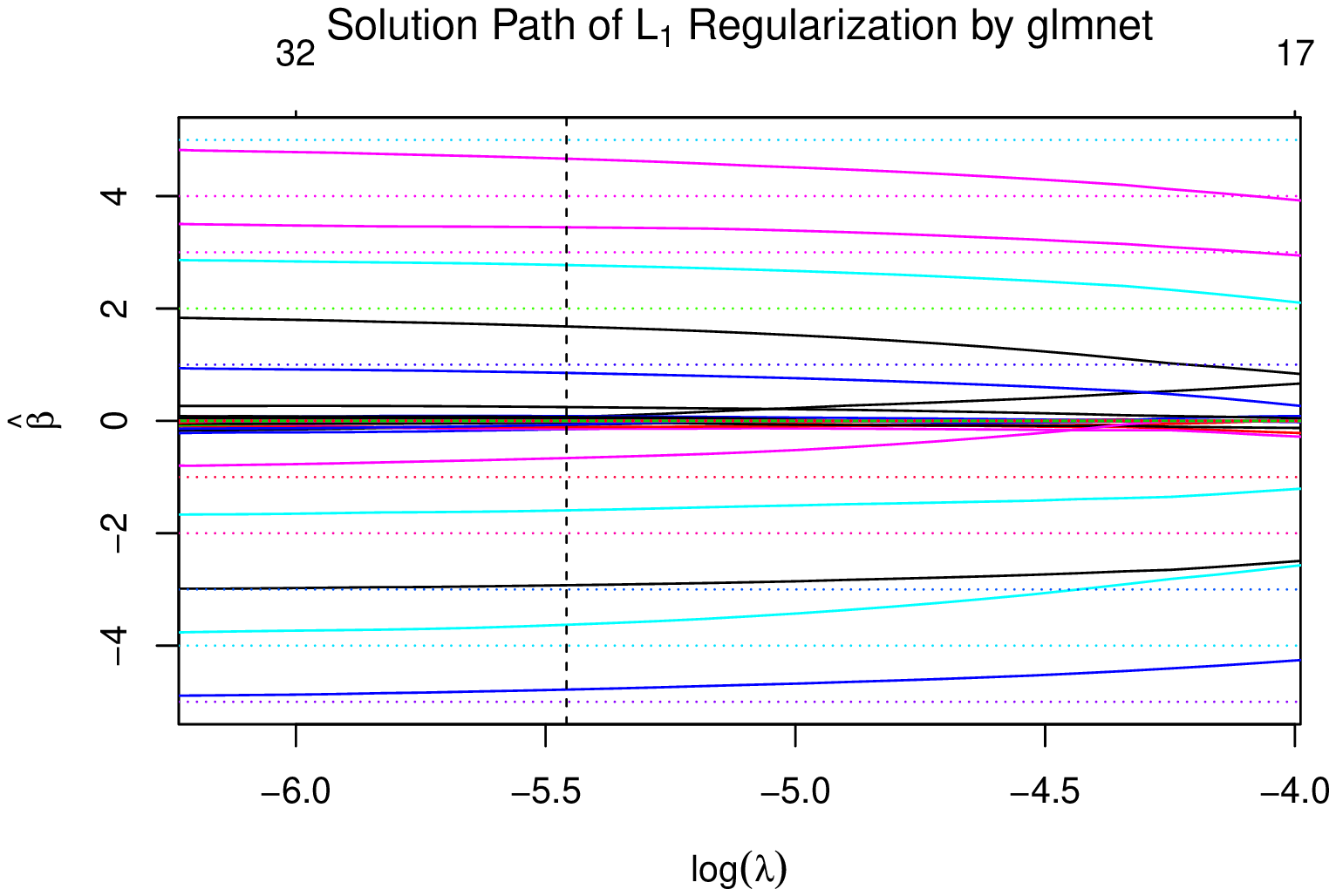}
   \end{minipage}
  \caption{Solution paths of $l_0$ regularization by SBR (left) and its $l_1$ counterpart Lasso by \texttt{glment} \citep{glmnet}.  \textit{The horizontal dotted lines indicate the true values of $\beta$, and the vertical dashed line indicates a $\hat\lambda^{CV}$ chosen by cross validation in $\log$-scale.  The range of $\lambda$ in both plots are from $\frac12\hat\lambda^{CV}$ to $4\hat\lambda^{CV}$.  Once $\lambda$ passes a certain value, the estimates of coefficients of $l_0$ regularization given by SBR is much more accurate and less sensitive to $\lambda$ than those of lasso given by \texttt{glmnet}.}}
\label{fig:solutionpath}
\end{figure}

In a large scale numerical experiment, we conducted $1000$ simulation trials with random $X, y$ generated as above.  We compare SBR with the ordinary least squares (OLS), two sparsity regularization methods, Lasso and elastic net with the tuning parameter $\alpha = 0.5$, and two Bayesian MCMC methods, Bayesian bridge \citep{polson2014} and Bayesian spike-and-slab shrinkage \citep{scott2014}, each with their state-of-the-art implementations.  The regularization parameter $\lambda$ in SBR is determined by a 10-fold cross validation.  For Lasso and elastic net, the regularization parameter $\lambda$ is chosen by the built-in cross validation in \texttt{glmnet} \citep{glmnet}.  For Bayesian bridge, we use the default prior in the package \texttt{BayesBridge} \citep{polson2012bridge}.  For spike-and-slab, the prior inclusion probability is set to be $0.5$, and other hyperparameters are the same as in the default setting in the package \texttt{BoomSpikeSlab} \citep{scott2016}.

Figure \ref{fig:mse} shows the accuracy in estimating $\beta$ for the six methods.  For OLS, Lasso, elastic net, and SBR, $\hat\beta$ are the minimizers of the corresponding optimization problems, and for Bayesian bridge and spike-and-slab, $\hat\beta$ are the posterior means.  SBR performs as well as the gold standard Spike-and-slab and better than all else, including the two widely used convex regularization methods.

\begin{figure}[!htb]
    \centering
    \includegraphics[scale = 0.4]{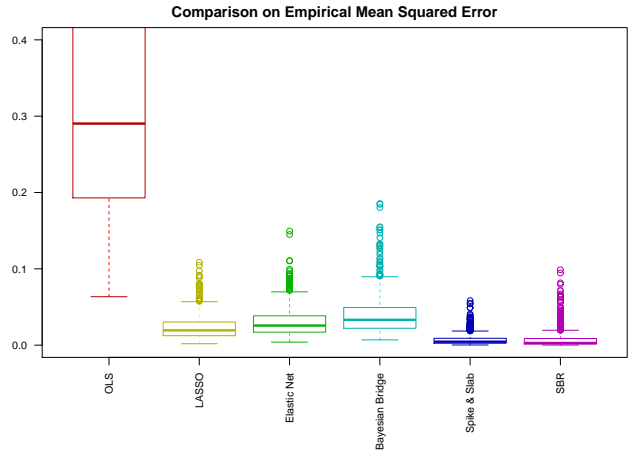}
    \caption{Comparison in estimation accuracy.  \textit{The boxplot depicts empirical mean squared errors of $1000$ simulation trials.  The $l_0$ regularization by SBR and Bayesian posterior mean estimators under the gold standard spike-and-slab priors outperform the convex regularization estimators in this regards.}}
    \label{fig:mse}
\end{figure}

In terms of variable selection, we compare SBR with the other $3$ sparsity-inducing methods, Lasso, elastic net, and Bayesian spike-and-slab.  For spike-and-slab, an coefficient is selected when its posterior inclusion probability is greater than $0.5$.  Figure \ref{fig:varsel} shows that all methods are able to select all the true non-zero coefficients almost all the time.  However, when compared in terms of preventing false selection, SBR and spike-and-slab are the best, whereas Lasso and elastic net tend to drastically over-select.

\begin{figure}[!htb]
    \centering
   \begin{minipage}{0.49\textwidth}
     \includegraphics[width=\linewidth]{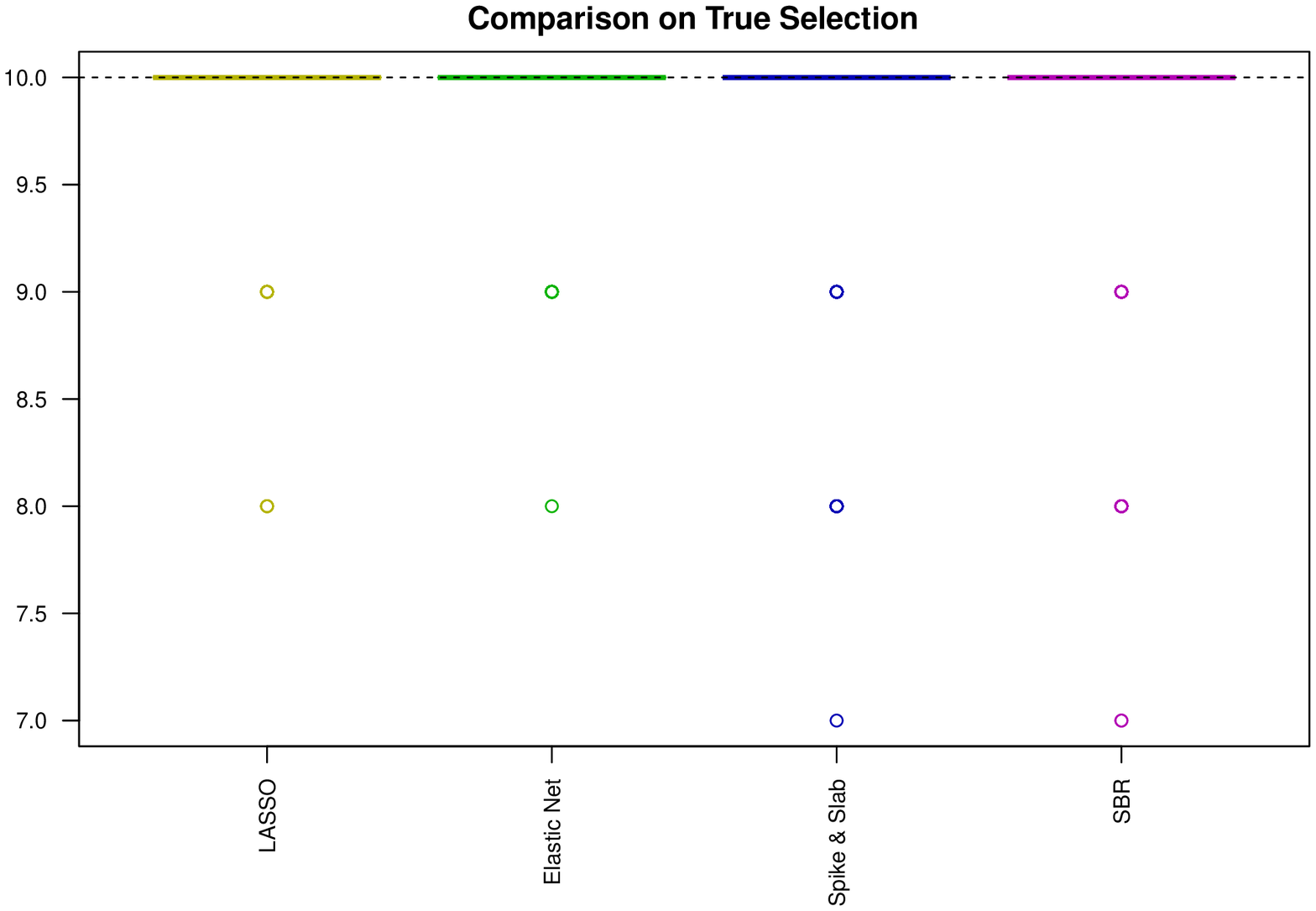}
   \end{minipage}
   \begin {minipage}{0.49\textwidth}
     \includegraphics[width=\linewidth]{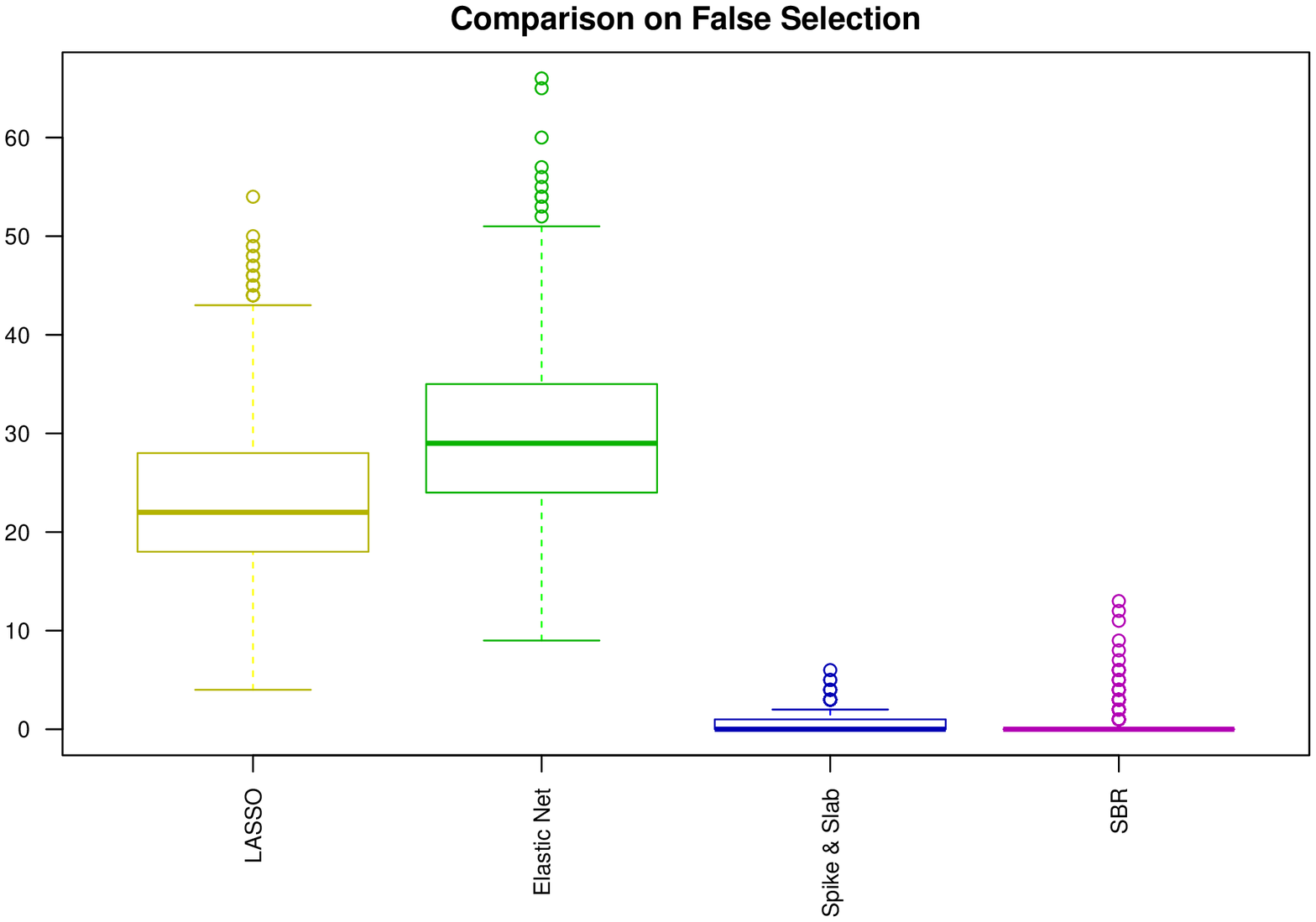}
   \end{minipage}
    \caption{Comparison in variable selection accuracy.  \textit{All $4$ methods generally select all of the $10$ (horizontal dotted line) true non-zero coefficients almost all the time.  SBR and spike-and-slab are on par of successfully preventing over-selecting, but Lasso and elastic net tend to produce a lot of false selections.}}
    \label{fig:varsel}
\end{figure}

Recently, \cite{hastie2017} compared Lasso with best subset selection \citep{bertsimas2016}, forward stepwise selection, and found similar phenomenon.  Namely, the performance of best subset selection and forward stepwise selection is overall similar, and both tend to outperform Lasso in the high signal-to-noise setting.  Since SBR is a forward-backward stepwise selection algorithm, it's no surprise then that SBR should give better results than Lasso in such setting.

\subsection{Computational efficiency and scalability of SBR}

In section \ref{perform} we've shown that the $l_0$ regularization has superior statistical properties in terms of both minimizing the estimation risk and selecting correct variables, especially its statistical performance improvement on convex regularization methods such as Lasso and elastic net.

Full Bayesian spike-and-slab performs very well statistically.  In order to achieve this good performance, spike-and-slab needs to do a complete MCMC sampling, and this task could take a significant amount of time, especially in high-dimensional settings, whereas regularization methods are usually able to handle large-scale computation efficiently.

In order to compare the computational efficiency of different methods, two sets of experiments, one with $n = 120, p = 100$, the other $n = 300, p = 200$, are run, and the results are plotted in Figure \ref{fig:time}.  SBR, as well as Lasso and elastic net, is almost as efficient as OLS, and only changes proportionally when the size of the problem increases.  On the contrary, the two full Bayesian methods, Bayesian bridge and especially spike-and-slab, are costly and scale badly with the problem size.  Actually when $n = 300, p = 200$, it could take as much as $40$ minutes to run even one spike-and-slab MCMC, whereas SBR finishes all $200$ simulation trials under $10$ seconds.  When $n$ and $p$ are in thousands, spike-and-slab is computationally intractable.

\begin{figure}[!htb]
\centering
\includegraphics[scale = 0.55]{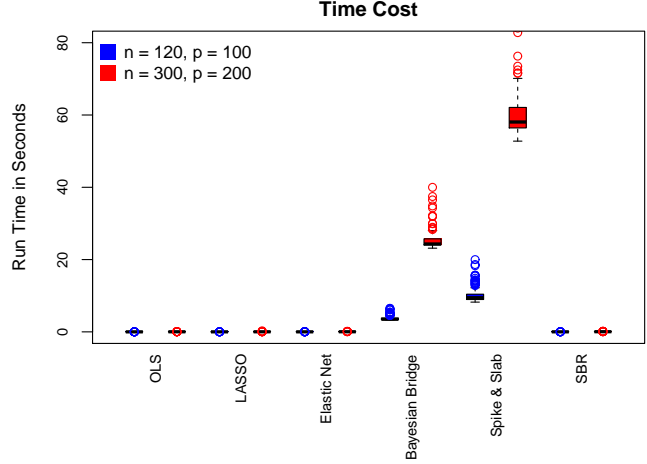}
  \caption{Time cost by different methods for different problem sizes.  \textit{SBR is as efficient as convex regularization methods Lasso and elastic net, whereas full Bayesian MCMC methods Bayesian bridge and spike-and-slab take significantly longer time.  When the problem size increases from $n = 120, p = 100$ to $n = 300, p = 200$, averaging over $200$ simulation trials, the time cost of Lasso changes from $0.008$ to $0.024$ seconds,  SBR $0.011$ to $0.046$ seconds, yet that of spike-and-slab MCMC surges from $10$ to more than $170$ seconds.}}
\label{fig:time}
\end{figure}

\subsection{Diabetes data \label{diabetes}}

Now we examine the performance of SBR on the classic diabetes data, available in the \texttt{R} package \texttt{lars} \citep{efron2004}.  The design matrix $X$ has $64$ columns, including all $10$ biochemical attributes and certain interactions.  Each column of $X$ has been normalized to have zero mean and unit $l_2$ norm.  The response $y$ is centralized to have zero mean.  We compare SBR with sparsity-inducing methods including Lasso, elastic net, and spike-and-slab priors, with the same settings as in Section \ref{perform}, and $\lambda$ determined by cross validation.  The results shown in Table \ref{table:sel} indicate that SBR's performance on variable selection is in line with popular sparse linear regression alternatives.

\begin{table}[h!]
\centering
\begin{tabular}{| C || C  | C | C | C |} 
 \hline
 Variable & Lasso & Elastic Net & Spike \& Slab & SBR\\
 \hline\hline
 \texttt{sex} 							& $\checkmark$ & $\checkmark$ & $\checkmark$ & $\checkmark$\\ 
 \texttt{bmi} 							& $\checkmark$ & $\checkmark$ & $\checkmark$ & $\checkmark$\\
 \texttt{map} 							& $\checkmark$ & $\checkmark$ & $\checkmark$ & $\checkmark$\\
 \texttt{hdl} 							& $\checkmark$ & $\checkmark$ & $\checkmark$ & $\checkmark$\\
 \texttt{ltg} 							& $\checkmark$ & $\checkmark$ & $\checkmark$ & $\checkmark$\\
 \texttt{glu} 							& $\checkmark$ & $\checkmark$ & --- 		 & ---			     \\
 $\texttt{age}^2$ 						& $\checkmark$ & $\checkmark$ & --- 		 & ---			     \\
 $\texttt{bmi}^2$ 						& $\checkmark$ & $\checkmark$ & --- 		 & ---			     \\
 $\texttt{glu}^2$ 						& $\checkmark$ & $\checkmark$ & --- 		 & $\checkmark$   \\
 $\texttt{age}\cdot\texttt{sex}$ 		& $\checkmark$ & $\checkmark$ & $\checkmark$ & $\checkmark$\\
 $\texttt{age}\cdot\texttt{map}$ 		& $\checkmark$ & $\checkmark$ & --- 		 & ---			\\
 $\texttt{age}\cdot\texttt{ltg}$ 		& $\checkmark$ & $\checkmark$ & --- 		 & ---			\\
 $\texttt{age}\cdot\texttt{glu}$ 		& $\checkmark$ & $\checkmark$ & --- 		 & ---			\\
 $\texttt{sex}\cdot\texttt{map}$ 		& ---		   & $\checkmark$ & --- 		 & ---			\\
 $\texttt{bmi}\cdot\texttt{map}$ 		& $\checkmark$ & $\checkmark$ & --- 		 & $\checkmark$ \\
\hline
\end{tabular}
\caption{Variable selection for the diabetes data.  \textit{Out of all $64$ variable, only those selected by at least one method are shown, chosen variables marked by $\checkmark$, and --- otherwise.  Lasso and elastic net select the most variables, whereas spike-and-slab the least.  Compared with these two extremes, SBR selects all variables chosen by spike-and-slab, but no variables not chosen by Lasso and elastic net.  The result indicates that SBR performs a reasonable variable selection task.}}\label{table:sel}
\end{table}

\section{Discussion \label{dis}}

Bayesian $l_0$ regularization can be solved using a fast and scalable single best replacement (SBR) algorithm.  In variable selection, this estimator possesses much of the statistical properties of spike-and-slab priors.  We provide theoretical links between the spike-and-slab MAP estimator and $l_0$ regularization.

We also explore the connection between regularized MAP estimators and posterior means \citep{strawderman2013}.  Tweedie--Masreliez construction of the posterior mean is re-interpreted as a proximal update rule.  This proximal update identity shows how the sparse posterior mode can be viewed as a posterior mean under a suitably defined prior.  Bernoulli-Gaussian (BG) and Bernoulli-Laplace (BL) priors are used for illustration.  Our approach demonstrates how regularized estimators can have good out-of-sample mean squared error.

In simulated and real data applications, SBR performs favorably compared with popular convex regularization methods such as Lasso and elastic net, as well as full Bayesian sampling methods including Bayesian bridge and spike-and-slab priors.

Recently non-convex feature selection methods for sparse signals estimation have gained increasing attention from the statistical learning community, including the classic SCAD penalty \citep{fan2001}, $l_q$ penalty \citep{marjanovic2013}, and horseshoe regularization \citep{bhadra2017horseshoe}.  There are a number of future directions for research, such as regularized logistic regression \citep{gramacy2012} and structural sparsity learning \citep{polson2017proximal}.  A comprehensive theoretical treatment and empirical comparison on different non-convex regularizations on the trade-off between statistical accuracy and computational efficiency remains open.

\bibliography{ref.bib}

\end{document}